\documentclass[wcp]{jmlr}

\usepackage{longtable}% for long tables
\usepackage{graphicx}
\usepackage{xcolor}

\newcommand{\expectancy}[3]{ \operatornamewithlimits{\mathbb{E}_{\:#1}^{\:#3}} \left [ #2 \right ] } 
 \newcommand{\D}{\mathcal{D}}

 \newcommand{\G}{\mathcal{G}}

 \newcommand{\x}{\mathbf{x}}
 \newcommand{\R}{\mathbb{R}}

\usepackage[]{algorithm}
\usepackage{amsmath}

\usepackage{booktabs}
\makeatletter
\let\Ginclude@graphics\@org@Ginclude@graphics 
\makeatother

\jmlrvolume{}
\jmlryear{xxx}
\jmlrworkshop{xxx}

\title[Short Title]{Early and Revocable Time Series Classification}

  \author{\Name{Youssef Achenchabe} \Email{youssef.achenchabe@universite-paris-saclay.fr}\\
  \addr Paris-Saclay University \\ Orange Labs
  \AND
  \Name{Alexis Bondu} \Email{alexis.bondu@orange.com}\\
  \addr Orange Labs
  \AND
  \Name{Antoine Cornuéjols} \Email{antoine.cornuejols@agroparistech.fr}\\
  \addr Paris-Saclay University
  \AND
  \Name{Vincent Lemaire} \Email{vincent.lemaire@orange.com}\\
  \addr Orange Labs
 }

\begin{document}

\maketitle

\begin{abstract}
Many approaches have been proposed for early classification of time series in light of its significance in a wide range of applications including healthcare, transportation and finance. 
Until now, the early classification problem has been dealt with by considering only irrevocable decisions.
This paper introduces a new problem called \textit{early and revocable} time series classification, 
where the decision maker can revoke its earlier decisions
based on the new available measurements. 
In order to formalize and tackle this problem, we propose a new cost-based framework and derive two new approaches from it. 
The first approach does not consider explicitly the cost of changing decision, while the second one does.
Extensive experiments are conducted to evaluate these approaches on a large benchmark of real datasets. 
The empirical results obtained convincingly show (\textit{i}) that the ability of revoking decisions significantly improves performance over the irrevocable regime, and (\textit{ii}) that taking into account the cost of changing decision brings even better results in general. 
\end{abstract}
\begin{keywords}
revocable decisions, cost estimation, online decision making
\end{keywords}

\section{Introduction}
\label{sec_intro}
%----------------------------------------------------------------------------------------------------------
%----------------------------------------------------------------------------------------------------------

Consider Eisenhower, in June 1944, having to decide when to launch the landing on the French coast (\cite{dday1944}). He had an imperfect knowledge of the weather conditions. The longer he waited, the more precise they became, allowing for a more informed decision: to launch the landing today or wait for another day, but the more difficult it became to ensure that all arrangements would be met and that the enemy remained unaware of the danger. Eisenhower was faced with a very common problem, even if dramatic here, to have to optimize a trade-off between the earliness of a decision and its potential cost. Note that once the decision to launch operation Overlord was made, it was irrevocable. There was no way it could be halted. 

In many situations, however, one can take a decision and then decide to change it after some new pieces of information become available. 
The change may be costly but still warranted because it seems likely to lead to a much better  outcome. 
This can be the case for instance when an outdoor event is canceled due to a dramatic change in the weather forecast, or when a doctor revises what now seems a misdiagnosis.  

%deciding to buy a house and gaining more information about it as time goes by,
%or to operate on a heart patient in a hospital.

Here again, we are faced with an an early classification problem: having to decide to make a prediction about a situation before 
all information becomes available because it is costly to wait, but now having the opportunity to change the prediction made if needed, potentially several times. 

We can liken our work to what is studied in control theory (\cite{bennett1993history}). Control theory is concerned as well with online decision making. For instance, when firing a rocket, the engines are controlled in strength and direction such as to maintain the rocket on the correct course. While the control here is instantaneous, other examples involve some prediction at what is likely to happen. This is the case with anti-aircraft guns that must point ahead of the plane in order to hit it. And the farther away is the plane from the gun, the more ahead of the plane should the pointing be. Notice that a whole lot of standard rules in control theory use as well integration of measurements over some time interval, thus taking stock of the past to take decisions.

Control theory is the province of smart engineers who know how to design mathematical formula that capture all of the knowledge pertinent for the problems to be solved. In numerous fields however, it is difficult or impossible to formulate such mathematical rules, either because the field is ripped with intricate and numerous factors, like in biology, economy or sociology, or because the environment changes in such a way that it is not worth trying to find mathematical formulas that would soon be useless. In these cases, one way to circumvent this difficulty is to rely on heuristic rules learned from data representative of the environment using machine learning methods. %This is what is classically done with prediction rules learned from time series, }
%\textcolor{magenta}{involving different types of Machine Learning models depending on the nature of the information to be predicted (e.g.: forecasting, classification, regression etc.).}

Early classification of time series goes one step further and tries to learn second order knowledge. The idea is to learn \textit{when} there will be enough information to decide in such a way as to optimize both the expectation of the misclassification cost and the cost associated with delaying the decision. 

Back to control theory, tremendous progress has been achieved when feedback signals have been taken into account. Then, the system, by observing the discrepancy between the prediction and further measurements, can correct its earlier decisions. Again, this is what is done in anti-aircraft systems. And, here also, these feedback loops which control by how much to correct, after which time interval, and so on, have to be designed by engineers using their knowledge of the domain.

In the early classification of time series context, this translates into revocable early classification of time series. There, the system is allowed to estimate the best time to make a decision, but also to estimate when to revise prior decisions, if the need seems warranted. There also, the rules can be learned from the available data. For instance, it can be anticipated that large amounts of such data are and will become available in the domain of autonomous vehicles. 

To illustrate in a more concrete way the usefulness of extending Early Classification to revocable decisions, let us consider the example of an emergency stop system implemented in an autonomous car. 
Let us assume that this system is equipped with several sensors, consisting of radars and cameras, which scan the road in order to detect a possible obstacle. The reliability of these measurements decreases when the distance between the car and the observed point increases.
The car is driven at 120 km/h on the highway and a dark shape located 300 meters ahead is detected on the camera image. At this moment, there is a doubt on the nature of this shape which could be an obstacle on the road with a probability of 0.008. Thus, the system decides to brake. While approaching, the image of the camera becomes more precise and the radars which have a range limited to 100 meters can now be used: it does not detect an obstacle. The system decides to release the brake because it recognizes a dark spot on the road which is at the origin of the shape initially observed.

%%%%%%%%%%%

%When the information comes as a series of measurements over time, this is the essence of \textit{early and revocable} time series classification. 
\smallskip
\noindent
\textbf{Some notations}

More formally, we assume that there exists a data set ${\cal S} = \{({\mathbf x}_T^i, y_i)\}_{1 \leq i \leq m}$ of \textit{complete} time series ${\mathbf x}_T \, = \, \langle {x_1}, \ldots, {x_T} \rangle$ each of which is associated with a label $y \in {\cal Y}$ (e.g. patient who needs a surgical operation or patient who does not). The measurements $x_{i}$ ${(1 \leq i \leq T)}$ belong to some input space ${\cal X}$ and can be univariate as well as multivariate. 
At each time step $t$, the decision-maker gets to know the time series measured so far: ${\mathbf x}_t \, = \, \langle {x_1}, \ldots, {x_t} \rangle$ and must decide either to make a prediction $\hat{y}_t$ about the class of the incoming time series or to postpone the decision. 

In the \textit{irrevocable regime}, once a decision has been taken, it cannot be changed and the decision-maker endures a cost which is the sum of the misclassification cost $\mathrm{C}_m(\hat{y}_t|y)$ plus the cost of having delayed the decision until time $t$: $\mathrm{C}_d(t)$. Whereas, in the \textit{revocable regime}, the decision-maker can change its prediction several times before the time limit $T$. Let us call $\D_{k}$, the sequence of the $k$ predictions $\langle \hat{y}_{t_1}, \ldots, \hat{y}_{t_k} \rangle$  made at times $t_1, \ldots, t_k$ in the time interval $\{1, \ldots, T\}$. Each decision change from $\hat{y}_{t_i}$ to $\hat{y}_{t_{i+1}}$ entails a cost $C_{cd}(\hat{y}_{t_{i+1}}|\hat{y}_{t_{i}})$. Then, the cost endured by the decision maker at time $T$ is given by Equation \ref{eq_total_cost} which sums all previously introduced costs.

%\begin{equation}
%g({\cal D}_k | {\mathbf{x}}_T, y) \; = \; \mathrm{C}_m \left (\hat{y}_{t_k}|y \right )  \, + \,  \mathrm{C}_d(t_k) \, + \, \sum\limits_{\substack{i=1 \\ \hat{y}_{t_{i}}, \hat{y}_{t_{i+1}} \in \cal D}_k  }^{k-1} C_{cd}(\hat{y}_{t_{i+1}}|\hat{y}_{t_{i}})
%\end{equation}

%\begin{equation}
%g({\cal D}_k | {\mathbf{x}}_T) \; = \; \mathrm{C}_m \left (\hat{y}_{\text{last}}|y \right )  \, + \,  \mathrm{C}_d({\hat{t}}_{\text{last}}) \, + \, \sum_{i=1}^{ |\D_{{\mathbf x_T}}|-1 } C_{cd}(\hat{y}_{i+1}|\hat{y}_{i})
%\end{equation}

%Again, the revocable regime is quite common. To illustrate it with a down to earth example,  let us suppose it is Spring and you want to organize a meal for a family gathering. You must decide wether it will be a picnic or a buffet in the house. You are quite anxious to make a success out of it and so, starting at dawn, you scrutinize the sky every half an hour in order to chose one of the two options. Depending on the aspect of the sky and its evolution, you may well end up taking decisions that you feel obliged to reverse afterwards at the cost of having to change the silverware and the tablecloths, with possibly exasperated yells from the spouse and children. But what should be the optimal strategy? If you wait to have enough information, you may end up with a low misclassification cost, but at the expense of having a high delay cost $\mathrm{C}_d(t)$. If, on the other hand, you make a decision early, you may have to change it later on and incur the cost of changing the decision plus the increased delay cost. 

Interestingly, while the early classification of time series, in the irrevocable regime, has been addressed in several papers in the last few years, we do not know of similar works for the revocable regime. 
One reason could be that this does not seem worthwhile. Are there so many situations, after all, where changing decisions could reduce the overall cost of the expected misclassification cost and of the cost associated with further delays? A second reason is that the problem seems difficult. Our work shows in effect that it is not straightforward to identify the best instants to revoke a decision. But yet, it shows also that such revocations can yield significant gains. This is demonstrated by the statistics presented at the beginning of the experimental section: indeed, there are very few useful revocations for sure, but the system presented is able to identify them. And, as shown by comparison between our approach and the best known irrevocable method, to our knowledge, the gain in performance is statistically significant. %(\textit{remplacerait le paragraph suivant}).

%%%%%%%%%%%%%%%%%%%%%%%%
\textcolor{magenta}{
%In addition, revoking decisions is a difficult problem that involves identifying useful decision changes and triggering them at the best times. The difficulty of this problem is illustrated by the statistics presented at the beginning of Section \ref{sec_Experiments}, which show that there are very few useful revocations. There is a high risk of choosing the wrong revocations and then degrading the performance of the system,} \textcolor{blue}{this has been verified in our experiments through the comparison with the irrevocable regime.
}

%%%%%%%%%%%%%%%%%%%%%%%%

\smallskip
The contribution of this paper is threefold. 
%First, it gives a definition of the optimization problem associated with taking possibly a sequence of decisions at minimal final cost when measurements are made over time. 
First, it formalizes the optimization problem associated with the revocable regime for the early classification problem.
Second, it proposes two approaches to tackle this problem. Both approaches are \textit{non-myopic} in that they take into account expectancies about likely futures to take their decisions. The first approach does not consider explicitly the cost of changing decision, while the second one does. 
Third, extensive experiments are presented that both allow the comparison of the two approaches and show that it is actually better to be able to revise decisions than to implement an irrevocable decision strategy. 

For the clarity of this paper, we consider a simple case where the input is in the form of a univariate time series whose measurements are observed over time (i.e., equivalent to a single sensor). The framework and approaches presented in this paper can be adapted to multivariate time series in a trivial way.

\smallskip
The rest of this paper is organized as follows.
Section \ref{related_work} provides an overview of classical  \textit{early classification} approaches, all of which deal with the irrevocable regime.
Section \ref{sec_ECONOMY_framework} focuses on a non-myopic framework which is designed for the irrevocable regime. The \textit{early and revocable classification} problem is defined in Section \ref{the_problem_to_solve}. Then, two new approaches are proposed, which are evaluated through extensive experiments in Section \ref{sec_Experiments}. %Section \ref{sec_Experiments}. 
Perspectives and future work are discussed in Section \ref{sec_conclusion}. 

\section{State of the art on early classification} 
\label{related_work}
All works that we are aware of deal with the early classification of time series problem in the irrevocable regime. Most of these works do not take into account explicitly the costs associated with a decision: the cost of misclassification and the cost of delaying decision. Instead, they generally base their decision on some form of confidence criterion and wait until a predefined threshold is reached before making a decision. 
For instance,   in (\cite{xing2009early}), the best time step to trigger the decision is estimated by determining the earliest time step for which the predicted label does not change, based on a 1NN classifier.  
    Similarly, (\cite{mori2017reliable}) proposes a method where the accuracy of a set of probabilistic classifiers is monitored over time, which allows the identification of time steps from whence it seems safe to make predictions.
In (\cite{parrish2013classifying,hatami2013classifiers,ghalwash2012early}), a classifier $h_t(\cdot)$ is learned for each time step and various stopping rules are defined (e.g. threshold on the confidence level).
    
%    Alternatively, interesting descriptors can be identified at start of time series, from a training set, in order to make early and reliable predictions. % can be reliable because they would be based on relevant similarities on the time series. %, those that would yield good levels of predictive performance. 
%    In \cite{xing2011extracting,ghalwash2014utilizing,he2015early} typical time series subsequences (i.e. shapelets) are identified in order  to distinguish classes as reliably as possible. 

With the advent of deep learning, researchers tried to revisit the problem of early classification in the light of these modern techniques. (\cite{russwurm2019end}) proposed a trainable framework for early classification of time series that can be fine-tuned end-to-end using standard gradient back-propagation. (\cite{martinez2018deep}) and (\cite{hartvigsen2019adaptive}) approach early classification as a reinforcement learning problem, a perspective extended to multi-label classification in (\cite{hartvigsenrecurrent}).

Few works do explicitly take costs into account.
A notable example is (\cite{mori2019early}) where the conflict between earliness and accuracy is explicitly addressed. Moreover, instead of setting the trade-off in a single objective optimization criterion as in (\cite{mori2017early}), the authors keep it as a multi-objective criterion and to explore the Pareto front of the multiple dominating trade-offs. 
%Accordingly, they propose a family of triggering functions for taking a decision involving hyperparameters to be optimized for each trade-off. 

It is however in (\cite{dachraoui2015early}), that the early classification problem is for the first time cast as the optimization of a loss function which combines the expected cost of misclassification at the time of decision plus the cost of having delayed the decision thus far. 
Importantly, besides the fact that this optimization criterion is well-founded, it permits also the expected costs for an incoming subsequence ${\mathbf x}_t$ to be estimated for future time steps.  A non-myopic decision procedure can thus be used. These expectations about the foreseeable future of an incoming time series can be learned from the training set {of $m$ full-length time series ${\cal S} = \{({\mathbf x}_T^i, y_i)\}_{1 \leq i \leq m}$.}

Approaches that do not explicitly consider costs are ill-equipped to deal with the possibility of revocable decisions. At best, they could base such revisions on observing that the confidence level falls below the pre-set threshold, and possibly exceeds it again, but this would not allow for the associated costs: of decision change and of delay. %, to enter the decision process. 

%\textcolor{magenta}{(At best, they could be based on the observation that the confidence level falls below the pre-set threshold, and eventually rises above it again,
%but this would not allow for the associated costs: of changing the decision and the delay in entering the decision process.)}

In the following, we therefore focus our attention on a setting where the costs are explicit factors entering the optimization problem.

%----------------------------------------------------------------------------------------------------------
%----------------------------------------------------------------------------------------------------------
\section{A cost-based non-myopic framework}
\label{sec_ECONOMY_framework}
%----------------------------------------------------------------------------------------------------------
%----------------------------------------------------------------------------------------------------------

In this section, we introduce a cost-based non-myopic framework that was designed for the irrevocable regime (\cite{achenchabe2021early}). The following sections will show how it can be adapted to the revocable regime. 

We suppose that a training  set ${\cal S} = \{({\mathbf x}_T^i, y_i)\}_{1 \leq i \leq m}$ of complete time series, with their associated labels, exists. 

\noindent
\textbf{I-} For each time step, $t \in \{1, \ldots, T\}$, a classifier $h_t$ can be learned $h_t: {\cal X}^t \rightarrow {\cal Y}$. 
%, from which a classifier $h_t$ has been learned for each time step $t \in \{1, \ldots, T\}$, $h_t: {\cal X}^t \rightarrow {\cal Y}$. 
Note that this contrasts with learning from data stream where the world can be non-stationary and the classifiers might have to evolve over time. 

\noindent
\textbf{II-} Using these classifiers and the knowledge that can be extracted from ${\cal S}$ when estimating the likely future of an incoming time series ${\mathbf x}_t$, it is possible to estimate the optimal instant for deciding a prediction about its class\footnote{This can be seen as an instance of the LUPI (Learning Under Privileged Information) framework (\cite{vapnik2009new}): during the learning phase, the learner has access to the full knowledge about the training time series ${\cal S} = \{({\mathbf x}_T^i, y_i)\}_{1 \leq i \leq m}$, while at testing time, only a subsequence ${\mathbf x}_t$ ($t < T$) is known.}.

%\textcolor{magenta}{We suppose that a training  set ${\cal S} = \{({\mathbf x}_T^i, y_i)\}_{1 \leq i \leq m}$ of complete time series, each provided of a single class label (for its full length), is available. Then the framework is made of two main steps. At the first step the training  set is used to train $T$ classifiers classifier. Classifiers, $h_t$, are  learned for each time step $t \in \{1, \ldots, T\}$, $h_t: {\cal X}^t \rightarrow {\cal Y}$.
%In the second step these classifiers
%are used to optimize a cost to trigger a decision with a partial information as described below .
%This setting contrasts with learning from data stream where the world can be non stationary and the labels might evolve over time. It can be cast in the LUPI (Learning Under Privileged Information) framework \cite{vapnik2009new}.}

%\smallskip
More precisely, given the \textit{misclassification cost} function $\mathrm{C}_m(\hat{y}|y) : {\cal Y} \times {\cal Y} \rightarrow \R$ that expresses the cost of predicting $\hat{y}$ when the true class is $y$ and the \textit{delay cost} function $C_d(t) : \R \rightarrow \R$ which is assumed to be an increasing function of time,
the expectancy of the cost of taking a decision at time $t$ given the incoming time series ${\mathbf x}_t$ is:

\begin{equation}
 \begin{split}
        \normalsize
        f(\mathbf{x}_t) \; &= \; \expectancy{(\hat{y},y) \in  {\cal Y}^2}{\mathrm{C}_m(\hat{y}| y)|\mathbf{x}_t}{t} \; +  \mathrm{C}_d(t)  \\
        &= \; \sum_{y \in {\cal Y}} P_{t}(y|\mathbf{x}_t) \, \sum_{\hat{y} \in {\cal Y}} P_t(\hat{y}|y, \mathbf{x}_t) \, \mathrm{C}_m(\hat{y}|y) \; + \; \mathrm{C}_d(t)
        \label{eq:cost1}
 \end{split}    
\end{equation}

where $\expectancy{(\hat{y},y) \in  {\cal Y}^2}{}{t}$ is the expectancy at time $t$, over the variables $y$ and $\hat{y}$. $P_{t}(y|\mathbf{x}_t)$ is the probability of the class $y$ given a time series that starts as ${\mathbf x}_t$ , and $P_t(\hat{y}|y, \mathbf{x}_t)$ is the probability that the classifier $h_t$ makes the prediction $\hat{y}$ given ${\mathbf x}_t$ as input and when $y$ would be its true label. 
In this non-myopic setting, the idea is that the decision of making a prediction is made at the current time $t$ only insofar that it is not expected that a lower cost could be achieved at a later time. 
This could happen if the expected misclassification cost would drop sufficiently to offset the increase of $\mathrm{C}_d(t)$. %in the delay cost. 

For any time in the future $t + \tau$ ($1 \leq \tau \leq T-t$), the expected cost of making a prediction can be estimated as:
\begin{equation}
 \begin{split}
        \normalsize
        f_{\tau}(\mathbf{x}_t) \; &= \; \expectancy{(\hat{y},y) \in  {\cal Y}^2}{\mathrm{C}_m(\hat{y}|y)}{t+\tau} \; +  \; \mathrm{C}_d(t+\tau) \;   \\
         &= \; \sum_{y \in {\cal Y}} P_{t+\tau}(y|\mathbf{x}_t) \, \sum_{\hat{y} \in {\cal Y}} P_{t+\tau}(\hat{y}|y, \mathbf{x}_t) \, \mathrm{C}_m(\hat{y}|y) \; + \; \mathrm{C}_d(t+\tau)
        \label{eq:cost2}
 \end{split}    
\end{equation}
\noindent
and $f_{0}(\mathbf{x}_t) \; = \; \expectancy{y \in  {\cal Y}}{\mathrm{C}_m(\hat{y}_t|y)}{t} \; +  \; \mathrm{C}_d(t)$ since we have access to predictions at current time.
%Because Equation \ref{eq:cost1} as well as Equation \ref{eq:cost2} cannot be directly estimated (the probability over a single event ${\mathbf x}_t$ cannot be estimated \textcolor{magenta}{?}), \cite{achenchabe2021early} propose to compute a partition of the times series given the training set ${\cal S}$, resulting in groups ${\mathfrak{g}_i} \in \G$.  We then get:
%%\begin{equation}
%%\small
%%\label{eq:cost1_1}
%%f(\mathbf{x}_t)=
%%\sum_{\mathfrak{g}_k \in \G} P_{t}(\mathfrak{g}_k|\x_t)  \sum_{y \in {\Y}} P_{t}(y|\mathfrak{g}_k)  
%% \sum_{\hat{y} \in {\Y}} P_{t}(\hat{y}|y,\mathfrak{g}_k) \mathrm{C}_m(\hat{y}|y) 
%%+ \mathrm{\mathrm{C}}_d(t)
%%\end{equation}
%%and:
%%\begin{equation}
%\begin{eqnarray}
%\small
%\label{eq:cost2_1}
%\lefteqn{f_{\tau}(\x_t)  \approx } \nonumber \\
%&  \sum\limits_{\mathfrak{g}_i \in \G} P_{t}(\mathfrak{g}_i|\x_t) 
%\sum\limits_{y \in {\Y}} P_{t}(y|\mathfrak{g}_i) 
% \sum\limits_{\hat{y} \in {\Y}} P_{t+\tau}(\hat{y}|y,\mathfrak{g}_i) \, \mathrm{C}_m(\hat{y}|y) \nonumber \\
% &  \, + \, \mathrm{C}_d(t+\tau)
%\end{eqnarray}
%
%\textcolor{magenta}{? VL: j'ai un soucis pour le passage de l'eq 3 à eq 4}
%\textcolor{magenta}{
%\begin{eqnarray}
%\small
%\label{eq:cost2_1}
%\lefteqn{f_{\tau}(\x_t)  = } \nonumber \\
%&  \sum\limits_{\mathfrak{g}_k \in \G} \left %( P_{t}(\mathfrak{g}_k|\x_t) 
%\sum\limits_{y \in {\Y}} %P_{t}(y|\mathfrak{g}_k) 
% \sum\limits_{\hat{y} \in {\Y}} P_{t+\tau}(\hat{y}|y,\mathfrak{g}_k) \, \mathrm{C}_m(\hat{y}|y) \right) \nonumber \\
% &  \, + \, \mathrm{C}_d(t+\tau)
%\end{eqnarray}
%}
%\end{equation}
%\noindent
Then the optimal decision time, at time $t$, is expected to be: 
%\begin{equation}
%\small
\label{eq:time2}
$\tau^* \; = \; \operatornamewithlimits{ArgMin}_{\tau \in \{0, \ldots, T-t\}} f_\tau(\mathbf{x}_t)$
%\end{equation}

The idea is to estimate the cost of a decision at all future time steps, up until $t=T$, based on the current knowledge about the incoming time series, and to postpone the decision to the time step that appears to be the best. 

If $\tau^* = 0$ then the best time for prediction seems to be now, the prediction $h_t({\mathbf x}_t)$ is returned and the classification process is terminated. Otherwise the decision is postponed to the next time step, and Eq. \ref{eq:time2} is computed again, this time with ${\mathbf x}_{t+1}$. The process goes on until a decision is made or $t = T$ at which point a prediction is forced. 

While this irrevocable decision process is well-founded and has proven to be quite efficient in extensive experiments (\cite{achenchabe2021early}), it can nonetheless lead to non optimal decisions when the estimated expected future cost of decision reveals itself to be erroneous. Figure  \ref{fig1} provides an example where at time $t$, it is expected that all future instants will lead to worse costs, when actually a better decision time occurs later, that could even have been anticipated given just a few additional measurements. 

\begin{figure}[htbp!]
\centering
\includegraphics[width=0.35\linewidth]{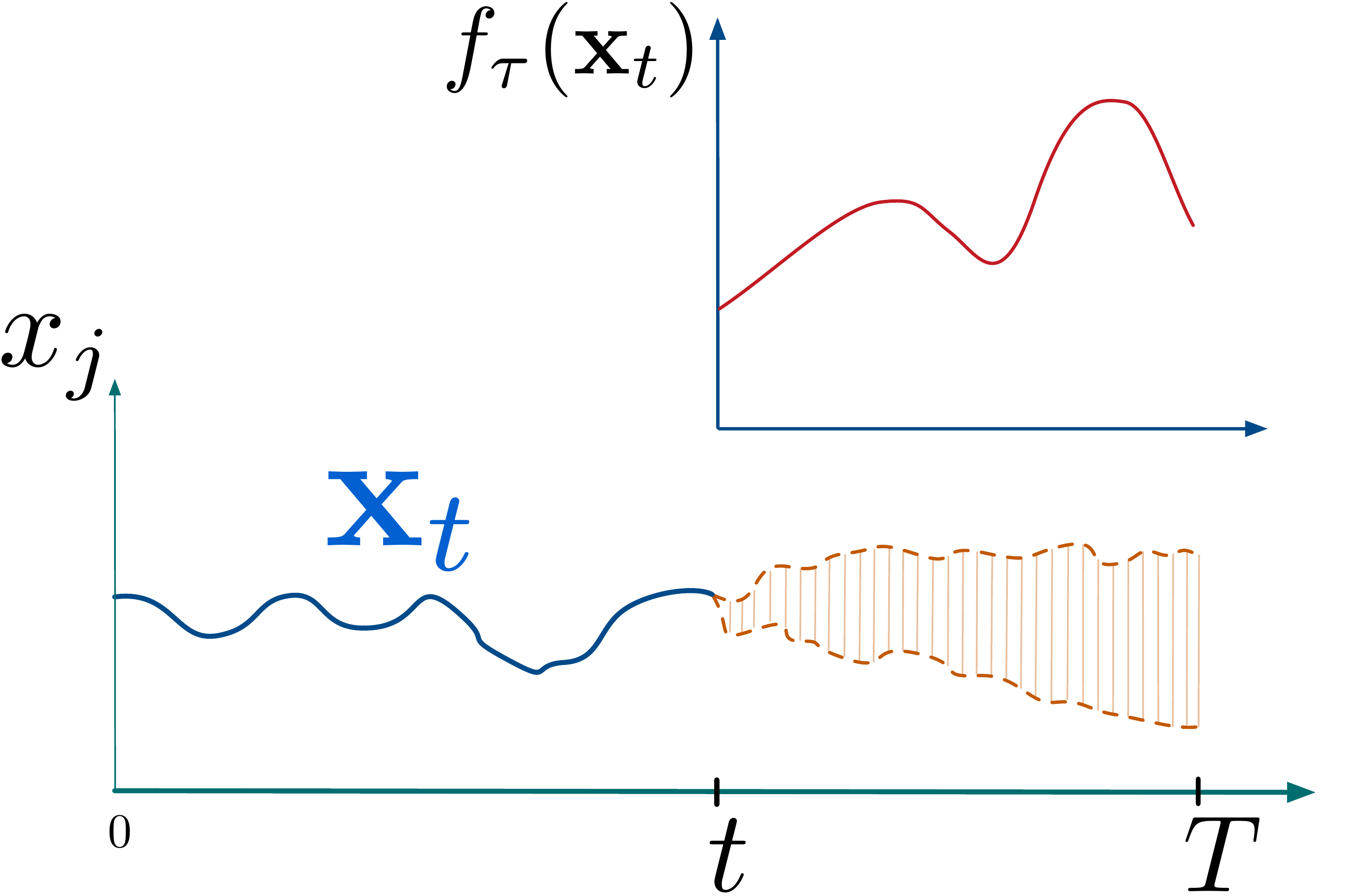}
\includegraphics[width=0.35\linewidth]{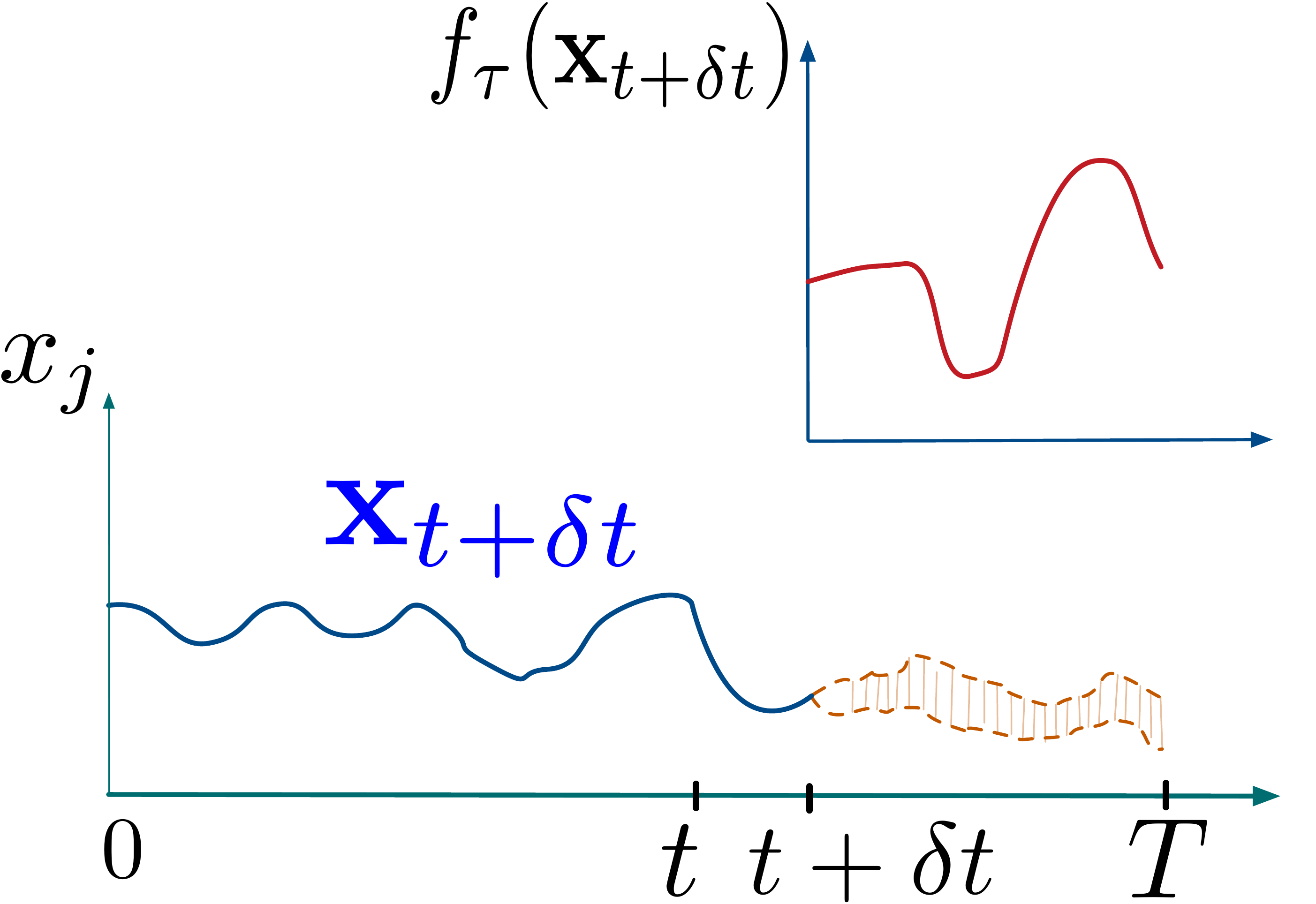}
\caption{(Left) At time $t$, the system foresees that this is the best time to make a decision, $f_\tau({\mathbf x}_t)$ is at its lowest. (Right) At some ulterior time $t+\delta t$, it appears that there should be a better time for decision in the future.}
\label{fig1}

\end{figure}

%\begin{figure}[htbp!]
%\centering
% \includegraphics[width=0.8\linewidth]{./figures/Fig1-prov.png}
%\caption{(a) .... (b) ... .}
%\label{fig1}
%\end{figure}

The question then arises as to how best adapt the irrevocable non-myopic strategy just described to the revocable regime where changes of decisions are allowed until $T$, but at the expense of incurring the additional costs associated with these changes.

%----------------------------------------------------------------------------------------------------------
%----------------------------------------------------------------------------------------------------------
 \section{A new framework for revocable decisions}
\label{the_problem_to_solve}
%----------------------------------------------------------------------------------------------------------
%----------------------------------------------------------------------------------------------------------

Suppose that while the measurements $x_t$ about time series ${\mathbf x}_T$ unfold from time $t=1$ to $t = T$, the decision-maker can change its mind as many times as it sees fit and  ends up triggering a sequence of predictions ${\cal D}_\ell = \langle \hat{y}_{t_1}, \ldots, \hat{y}_{t_{\ell}} \rangle$ about the class of the input %incoming 
time series.
The final cost incurred will be:
\begin{equation}
g({\cal D}_\ell | {\mathbf{x}}_T, y) \; = \; \mathrm{C}_m \left (\hat{y}_{t_\ell}|y \right )  \, + \,  \mathrm{C}_d(t_\ell) + \hspace{-0.3cm} \sum\limits_{\substack{i=1 \\ \hat{y}_{t_{i}}, \hat{y}_{t_{i+1}} \in \cal D}_\ell   }^{\ell-1}\hspace{-0.3cm}C_{cd}(\hat{y}_{t_{i+1}}|\hat{y}_{t_{i}})
\label{eq_total_cost}
\end{equation}
where $t_\ell$ is the time of the last change of decision yielding the prediction $\hat{y}_{t_\ell} = h_{t_\ell}({\mathbf x}_{t_\ell})$, when the true class of ${\mathbf x}_T$ %was
is $y$. Moreover, the cost of changing decision is defined as $\mathrm{C}_{cd}(\hat{y}_1|\hat{y}_2) : {\cal Y} \times {\cal Y} \rightarrow \R$ with $\mathrm{C}_{cd}(\hat{y}_1|\hat{y}_2)=0$ if $\hat{y}_1=\hat{y}_2$.

%It is obvious that the more changes of decision are performed, the costlier is the second term in Equation \ref{eq_total_cost}. This can only be offset if the final prediction $\hat{y}_k$ yields a low cost of misclassification and if, in addition, the delay cost $\mathrm{C}_d(t_k)$ is low enough.

\smallskip
Formally, the problem is to find a sequence of decisions %changes 
${\cal D}^\star \in {\mathbb{D}_T}$ that minimizes Equation \ref{eq_total_cost}:
\begin{equation}
{\cal D}^\star \; = \; \operatornamewithlimits{ArgMin}_{{\cal D} \in {\mathbb{D}_T}} g({\cal D} | {\mathbf{x}}_T, y) 
\end{equation}
where $\mathbb{D}_T$ is the set of all possible sequences of maximum length $T$.

\smallskip
When, at time $t$, only a partial knowledge ${\mathbf x}_t$ is available about the incoming time series, Equation \ref{eq_total_cost} cannot be computed. A sequence of decisions  ${\cal D}_k = \langle \hat{y}_{t_1}, \ldots, \hat{y}_{t_k} \rangle$ has been taken so far, and the question is to see if changing the last decision $\hat{y}_{t_k}$ now, at time $t$, is favorable, because it would bring a better expected cost, and it would not seem better to postpone such a possible change to a later time $t+\tau$. 

The cost of adding a new decision at time $t+\tau$, can be estimated as:

\begin{equation}
 \begin{aligned}
f_\tau^{\text{rev}}({\cal D}_k, \Tilde{t} \,|\,{\mathbf x}_t) \; = \; 
&\expectancy{(\hat{y},y) \in  {\cal Y}^2}{\mathrm{C}_m(\hat{y}|y) |\, {\mathbf x}_t}{t+\tau} \, + 
\sum\limits_{\substack{i=1 \\ \hat{y}_{t_{i}}, \hat{y}_{t_{i+1}} \in \cal D}_k  }^{k-1} \mathrm{C}_{cd}(\hat{y}_{t_{i+1}}|\hat{y}_{t_{i}}) \, 
+ \, 
\, \\ &\expectancy{\hat{y} \in  {\cal Y}}{\mathrm{C}_{{cd}}(\hat{y}|\hat{y}_{t_k})|\, {\mathbf x}_t}{t+\tau} \, + 
\, \vphantom{\expectancy{\hat{y}_t,y}{\mathrm{C}_m(\hat{y}_t|y)}{t+ \tau}} \mathrm{C}_d(\Tilde{t})
\end{aligned}
\label{eq:change_cost_criterion}
\end{equation}

and for $\tau=0$, \\$f_{\tau=0}^{\text{rev}}({\cal D}_k, \Tilde{t} \,|\,{\mathbf x}_t) \; = \; 
\expectancy{y \in  {\cal Y}}{\mathrm{C}_m(\hat{y}_t|y) |\, {\mathbf x}_t}{t} \, +
\sum\limits_{\substack{i=1 \\ \hat{y}_{t_{i}}, \hat{y}_{t_{i+1}} \in \cal D}_k  }^{k-1} \mathrm{C}_{cd}(\hat{y}_{t_{i+1}}|\hat{y}_{t_{i}}) \, 
+ \mathrm{C}_d(\Tilde{t})$\\
where $\Tilde{t}$ is introduced as input of the function $f_\tau^{\text{rev}}$ in order to control the delay cost paid by the user. In addition, the expected cost of changing decision is defined as follows for ($1 \leq \tau \leq T-t$): 

\begin{equation}
\expectancy{\hat{y} \in  {\cal Y}}{\mathrm{C}_{{cd}}(\hat{y}|\hat{y}_{t_k})|\, {\mathbf x}_t}{t+\tau} \; = \; \sum_{\hat{y} \in {\cal Y}}    P_{t+\tau}(\hat{y}|\hat{y}_{t_k},{\mathbf x}_t) \mathrm{C}_{cd}(\hat{y}|\hat{y}_{t_k}) 
\label{eq:change_cost_expectancy}
\end{equation}
and for $\tau=0$, this term equals zero by convention because we have full knowledge of the prediction at the current time step t. Then, given that the notation ${\cal D}_{k+1}$ is used to denote the sequence of decisions $\langle \hat{y}_{t_1}, \ldots, \hat{y}_{t_k}, \hat{y}_{t} \rangle$, the criterion for changing decision at time $t$ becomes:

%\approx 
%  \sum_{\mathfrak{g}_k \in \G} P_{t}(\mathfrak{g}_k|\x_t) \expectancy{}{\mathrm{C}_{cd}|\mathfrak{g}_k, \hat{y}_{t_k}}{t+\tau}

\begin{equation}
%\small
\label{eq:change_cost}
%\hat{t}_2 \; = \; \operatornamewithlimits{Min}_{t \in \{\hat{t}_1, \ldots, T\} }t ~~~\text{st.} ~~~
\text{\textit{criterion}} \; = \; 
\begin{cases}
~~~~~~~~~ ~~~ \hat{y}_{t} \neq \hat{y}_{t_k} \\
\text{and}~~~ \operatornamewithlimits{ArgMin}\limits_{\tau \in \{0, \ldots, T-t\}} f_\tau^{\text{rev}}({\cal D}_k, t+\tau \, | \, \mathbf{x}_{t}) \; = \; 0 \\
\text{and}~~~ f_{\tau=0}^{\text{rev}}({\cal D}_{k+1}, t \,|\,{\mathbf x}_t)  < f_{\tau=0}^{\text{rev}}( {\cal D}_k, t_k\,|\,{\mathbf x}_{t})
\end{cases}
\end{equation}

%
%Note that $f_\tau^{\text{rev}}({\mathbf x}_t, {\cal D}_k)$ is a minimum over the cost of all sequences of decisions \textcolor{orange}{(*)} that could happen between $t$ and $t+\tau$ because terms corresponding to the added changes of decision and their associated costs would be added to Equation \ref{eq:change_cost_criterion} which supposes no intervening decision between $t$ and $t+\tau$. 
%

A decision is thus taken at time $t$ only if ($i$) the current prediction $\hat{y}_{t}$ would differ from the last one $\hat{y}_{t_k}$,  (\textit{ii}) if it seems that now is the best time to make a new decision, and (\textit{iii}) if the estimated cost with the new prediction would be less than the engaged one with the previous decision.

%prediction.

\smallskip
An interesting case occurs \textit{when changing decision is costless}:  {\normalsize $\forall y, y' \in {\cal Y} \times {\cal Y}, ~\mathrm{C}_{cd}(y | y') \, = \, 0$}.  Equation \ref{eq:change_cost_criterion} becomes: 

\begin{equation}
 \begin{aligned}
f_\tau^{\text{rev}}({\cal D}_k, \Tilde{t}\,|\,{\mathbf x}_t) \; = &\; 
\expectancy{(\hat{y},y) \in  {\cal Y}^2}{\mathrm{C}_m(\hat{y}|y) |\, {\mathbf x}_t}{t+\tau} \, + 
\, \vphantom{\expectancy{\hat{y}_t,y}{\mathrm{C}_m(\hat{y}_t|y)}{t+ \tau}} \mathrm{C}_d(\Tilde{t})
\end{aligned}
\label{eq:change_nullcost_criterion}
\end{equation}
which is Equation \ref{eq:cost2}. Then, the strategy is to change decision each time the gain in the expected misclassification cost with a new decision offsets the increased delay cost.

\medskip
\noindent
Now a question is: what would be the \textbf{optimal sequence of decisions} if the decision maker had access to the true class $y$ of the incoming time series but could only use the classifiers $h_t (t \in \{1, \ldots, T\})$ to make its prediction? (Thus, it could not output $y$ before the first time $t ~ \text{st.} ~ h_t({\mathbf x}_t) = \hat{y}_t = y$).

\begin{theorem}[Optimal sequence of decisions when $\forall (y,y') \in {\cal Y}^2$, $C_{cd}(y|y') > 0$]
\label{th_optimal_seq}
%~\\
For \newline
any time series ${\mathbf x}_T$ of class $y$, the optimal sequence of decision is reduced to a one decision sequence where the (or one of possibly several) optimal time(s)
$t^\star$ is defined by: $t^\star \; = \; \operatornamewithlimits{ArgMin} _{1 \leq t \leq T} \bigl\{  \mathrm{C}_m(\hat{y}_t|y) \, + \, \mathrm{C}_d(t) \bigr\}$.  
\end{theorem}

\begin{proof}
Consider a sequence of decisions ${\cal D}_k = \langle \hat{y}_{t_1}, \ldots, \hat{y}_{t_k} \rangle$ taken at times $\{t_1, \ldots, t_k\}$. Then the cost paid at time $T$ is: $\sum_{i=1}^{k-1}\mathrm{C}_{cd}(\hat{y}_{i+1}|\hat{y}_{i }) \, + \, \mathrm{C}_m(\hat{y}_k|y) \, + \, \mathrm{C}_d(t_k)$ which cannot be less than: $\mathrm{C}_m(\hat{y}_{t^\star}|y) \, + \, \mathrm{C}_d(t^\star)$.
\end{proof}

%\textcolor{magenta}{on pourrait aller plus loin dans la thèse de youssef : dès que Cm decroit moins vite que Cd alors ce point est point d'ultime décision dans le mode non révocable}

%\textcolor{magenta}{Vincent en est là pour la section IV}

Theorem \ref{th_optimal_seq} shows that it is better to make the optimal decision at the right time rather than revoking a decision since this can only lead to sub-optimal sequences of decisions. However, in practice, without having access to the ground truth $y$, it may be beneficial to make a first guess and to change it later on. 

In a consistent way, $f_\tau^{\text{rev}}$ envisions a single future decision $\hat{y}_{t+\tau}$, which entails a minimal cost compared to longer sequences involving several decision changes between $t$ and $t+\tau$. Indeed, the costs associated with the successive decision changes would be added to the Equation \ref{eq:change_cost_criterion}, and would necessarily lead to a higher total cost.

%It thus appears that the revocable regime, which may produce sequences of several decisions, is interesting only insofar as the knowledge of the incoming time series, and hence its class, is imperfect for $t < T$. This may indeed cause one or several re-examinations of the predicted class $\hat{y}$ as knowledge about the incoming time series is gained. 

\medskip
\noindent
It must be noted that the \textit{criterion} (Equation \ref{eq:change_cost}) does not specify how and when to make \textbf{the first prediction} $\hat{y}_{t_1}$. Since a decision is mandatory in the framework of decision making, we assume that a ``no decision'' is associated with an infinite cost: 
%In the case of early classification it seems no sense in not making a decision. 
%Thus we define in case of no decision an expected cost as infinite: 
$f_\tau^{\text{rev}}(\emptyset \, | \,{\mathbf x}_t) = + \, \infty$. By contrast, when a first decision is taken, its  expected cost is:

\begin{equation*}
f_{\tau=0}^{\text{rev}}(\langle \hat{y}_1\rangle \,, t, |\,{\mathbf x}_t) = \; 
 \expectancy{y \in  {\cal Y}^2}{\mathrm{C}_m(\hat{y}_t|y) | {\mathbf x}_t}{t}  + 
 \vphantom{\operatornamewithlimits{Max}_{y, \hat{y}}{ \mathrm{C}_m(\hat{y}|y)}} \; 0 \;  + 
 \vphantom{\operatornamewithlimits{Max}_{y, \hat{y}}{ \mathrm{C}_m(\hat{y}|y)}} \mathrm{C}_d(t)
\end{equation*}

Consequently,
%Since
$f_{\tau=0}^{\text{rev}}(\langle\hat{y}_1\rangle \,|\,{\mathbf x}_t) \, < \, f_{\tau=0}^{\text{rev}}(\emptyset \, | \,{\mathbf x}_t)$ %always holds, and the cost of changing the decision is null at this time,
and the first decision is made accordingly to the non-myopic strategy defined by $f_\tau({\mathbf x}_t)$ in its irrevocable regime (see Equation \ref{eq:cost1}).

%\textcolor{magenta}{on pourrait se passer de la partie magenta on parlant de l'équation 8, il existe un état initial sans décision (step1 yt=vide et il n'y aurait alors que le 3è terme à changer en mettant inférieur ou égale) ???}

%Then: $\hat{g}(\{\hat{y}_1\} | {\mathbf{x}}_t) \, < \, \hat{g}(\emptyset | {\mathbf{x}}_t)$, and \textit{criterion 1} as well as \textit{criterion 2} reduce to their first and second conditions, which means that the first decision should be made according to the non myopic strategy in its irrevocable regime. 

%----------------------------------------------------------------------------------------------------------
%----------------------------------------------------------------------------------------------------------
%\section{Proposed approaches}
%\label{sec_proposed_approaches}
%----------------------------------------------------------------------------------------------------------
%----------------------------------------------------------------------------------------------------------

%In this section, we present first a generic algorithmic implementation of the revocable decision-making criterion as defined in Equation \ref{eq:change_cost}. This is described in algorithm \ref{algo:algo2}. 

\bigskip
A generic algorithmic implementation of the revocable decision-making criterion as defined in Equation \ref{eq:change_cost} is presented in Algorithm \ref{algo:algo2}.
The sequence of decisions is initialized with $\langle \hat{y}_{t_1} \rangle$, where $\hat{y}_{t_1}$ is provided using the irrevocable strategy based on $f_\tau({\mathbf x}_t)$ (see Equation \ref{eq:cost1}) and $t_1$ is the corresponding triggering time. 
The cost of the previous decision $cost_{prev}$ is initialized with the cost of the first decision  $f_{\tau=0}({\mathbf x}_{t_1})$. A new decision is triggered if the conditions of Equation \ref{eq:change_cost} are satisfied.

\smallskip
One goal of our research is to evaluate the added value of explicitly taking into account the cost of the changes of decision with respect to a revocable strategy which would not. % (i.e. the changes of decision are considered as costless as in Equation \ref{eq:change_nullcost_criterion}). 
Accordingly, we coded two algorithms. 

\begin{enumerate}
   \item The first one is named \textsc{eco-rev-cu} for cost unaware (as in Equation \ref{eq:change_nullcost_criterion}).
   \item The second is named \textsc{eco-rev-ca} for cost aware (as in Equation \ref{eq:change_cost_criterion}). 
\end{enumerate}

%----------------------------------------------------------------------------------------------------------
%----------------------------------------------------------------------------------------------------------

\begin{algorithm}
%\footnotesize
    \SetKwInOut{Input}{input}
    \SetKwInOut{Output}{output}

    $decisions \gets \langle \hat{y}_{t_1} \rangle$

    {$t_{prev} \gets t_1$}

    \For{$t= t_1+1\dots T$ }{
      
            $\tau^{\star} \gets ArgMin_{\tau \in \{0 \dots T-t\}} f^{\text{rev}}_{\tau}(\text{decisions} \,, t+\tau | \, \mathbf{x}_{t})$
            
            $cost_{new} \gets f^{\text{rev}}_{\tau=0}(\text{decisions} \,, t+\tau^* | \, \mathbf{x}_{t})$\\
             {$cost_{prev} \gets f^{\text{rev}}_{\tau=0}(\text{decisions} \,, t_{prev} | \, \mathbf{x}_{t})$}\\
            \uIf{$\hat{y}_t \neq \hat{y}_{t_{_{prev}}} \textbf{and} \tau^* = 0  \textbf{and} cost_{new} < cost_{prev}$ }{
                $t_{prev} = t$ \\
                $\text{decisions} \gets \text{decisions} \cup \hat{y}_t$ \; 
              }
    }
    $return decisions$ 
    
  \vspace{4mm}    
    \caption{Generic revocable regime algorithm}
    \label{algo:algo2}

\end{algorithm}

\section{Experiments}
\label{sec_Experiments}
%----------------------------------------------------------------------------------------------------------
%----------------------------------------------------------------------------------------------------------

%\textcolor{magenta}{ 
%In the next section, we propose two new approaches in order to answer three questions: 
%\begin{itemize}
%    \item Is it possible to identify the useful revocations and reach better performance?
%    \item By how much being informed with the values of $C_{cd}$ allows to improve revocations?
%    \item What is the effect of the costs $C_d$ and $C_{cd}$ on the two previous questions ? 
%\end{itemize}
%}

\label{expe}

%decribe new the Avg cost evaluation criterion (un choix parmi d'autres, temps dans chaque état, nombre de décisions ... finalement, c'est dépendent du cas d'usage, on a fait un choix simple)

%\begin{itemize}
%    \item $C_{cd}$
%    \begin{itemize}
%        \item very low $C_{cd}$ = \{ 0.0001, 0.00025, 0.0005, 0.00075\}
%        \item low $C_{cd}$ =  \{ 0.001, 0.0025, 0.005, 0.0075\}
%        \item medium $C_{cd}$ =  \{ 0.01, 0.025, 0.05, 0.075\}
%        \item high $C_{cd}$ =  \{ 0.1, 0.25, 0.5, 0.75, 1\}
%    \end{itemize}
    
%    \item time cost
%    \begin{itemize}
%        \item very low $C_{d}$ = \{ 0.0001, 0.00025, 0.0005, 0.00075\}
%        \item low $C_{d}$ =  \{0.001, 0.0025, 0.005, 0.0075\}
%        \item medium $C_{d}$ =  \{0.01, 0.025, 0.05 ,0.075\}
%        \item high $C_{d}$ =  \{0.1, 0.25, 0.5, 0.75, 1\}
%    \end{itemize}
%\begin{itemize}

%    \item Goal of the experiments
%    \item Data and feature extraction: same as \cite{achenchabe2021early}
%    \item Evaluation criterion

%    \item Baseline: Economy-$\gamma$
%    \item Show results on Variant B2
%    \item Analysis
%    \item Front de pareto $C_cd$
%    \item Dans la partie résultat, mentionner que l'approche qui prend en compte explicitement l'hisorique des changement de décision est moins bonne (trop prudente), cf suplémentary material. 
%\end{itemize}

The criterion defined by Equation \ref{eq:change_cost} and its implementation in Algorithm \ref{algo:algo2} aim at identifying the advantageous changes of decision, those that take advantage of the knowledge gained with additional measurements of the incoming time series. %to improve the trade-off between accuracy and earliness. 
With the experiments, we want to measure the true added value of this strategy. Specifically, the question is twofold. First, does it recognize useful changes of decisions: those that increase the performance? Second, does it pay off to implement a revocable strategy that takes into account the costs of changing decisions by comparison to one that would not consider these costs? In the following, we report results obtained on 34 datasets (see Section \ref{sec_data}) for a whole range of values for the delay cost $\mathrm{C}_d$ and the cost incurred if changing decision $\mathrm{C}_{cd}$.

%----------------------------------------------------------------------------------------------------------
\subsection{Implementation choices}
%----------------------------------------------------------------------------------------------------------

Equation \ref{eq:cost2} for the irrevocable strategy has been proposed and has given way to several different algorithmic versions generically called  \textsc{Economy} as described in (\cite{achenchabe2021early}). 
They differ in the way they group time series in order to estimate the expectation $\expectancy{(\hat{y},y) \in  {\cal Y}^2}{\mathrm{C}_m(\hat{y}|y)}{t+\tau}$. 
Of all these methods, \textsc{Economy}-$\gamma$ is the one that stands out, both because its refined way of predicting the likely future of an incoming time series and its significantly better performances demonstrated in extensive experiments over the other \textsc{Economy} versions as well as with the method of (\cite{mori2017early}). 
This is why it is the method used in our experiments, both in a revocable version that takes the cost of changing decisions into account, and one that does not.

\paragraph{Implementation of the two proposed approaches:} in our experiments, the \textsc{eco-rev-cu} algorithm is simply the \textsc{Economy}-$\gamma$ algorithm allowed to be reiterated after each decision. It thus does not take into account the costs associated with changing decisions, whereas \textsc{eco-rev-ca} does. 
More technically, \textsc{eco-rev-ca} approximates $\expectancy{\hat{y} \in  {\cal Y}}{\mathrm{C}_{{cd}}(\hat{y}|\hat{y}_{t_k})|\, {\mathbf x}_t}{t+\tau}$ in Equation \ref{eq:change_cost_expectancy} by using the groups of time series, denoted by $\G$:

\begin{equation}
\label{eq:cost6}
\normalsize
\expectancy{\hat{y} \in  {\cal Y}}{\mathrm{C}_{{cd}}(\hat{y}|\hat{y}_{t_k})|\, {\mathbf x}_t}{t+\tau}  \approx 
  \sum_{\mathfrak{g}_k \in \G} P(\mathfrak{g}_k|\x_t) \expectancy{\hat{y} \in  {\cal Y}}{\mathrm{C}_{{cd}}(\hat{y}|\hat{y}_{t_k})|\, \mathfrak{g}_k}{t+\tau}
\end{equation}

\noindent
Then, the probability $P_{t+\tau}(\hat{y}|\hat{y}_{t_k},\mathfrak{g}_k)$ is estimated in a frequentist way as the proportion of time series predicted to belong to $\hat{y}_{t_k}$ at time ${t_k}$, and for which the classifier changed its decision at time $t+\tau$ by predicting the class $\hat{y}$.
For full reproducibility of the experiments presented in this paper, an open-source code is available in the supplementary material. %\cite{GIT}. TODO : remettre le ref pour la camera ready 
%at {https://github.com/eco-rev/rev-algo}.

\paragraph{Overview of the \textsc{Economy}-$\gamma$ approach:} 
more precisely, the groups $\G$ are obtained by stratifying the time series by confidence levels\footnote{This restricts these methods to binary classification problems.} of $h_t$. At each time step $t$, the confidence level $p(h_t(\x_t) = 1)$ of the classifier can take a value in $[0,1]$. Examining the confidence levels for all time series in the validation set ${\cal S'}^t$ truncated to the first $t$ observations, we can discretize the interval $[0, 1]$ into $K$ equal frequency intervals, denoted $\{I_t^1,\ldots, I_t^K\}$.
%For instance, if $K=5$, and $|{\cal S'}^t| = 1000$, the intervals  $I_t^1 = [0, 0.30[$, $I_t^2 = [0.30, 0.45[$, $I_t^3 = [0.45, 0.58[$, $I_t^4 = [0.58, 0.83[$, $I_t^5 = [0.83, 1]$ could each correspond to 200 training time series. The discretization of confidence levels into equal frequency intervals corrects any bias in the calibration of $h_t$, in a similar way to isotonic calibration \citep{flach2016classifier}.
Then, the future expectancy $\expectancy{y \in  {\cal Y}^2}{\mathrm{C}_m(\hat{y}_t|y) | {\mathbf x}_t}{t+\tau}$ is estimated by modeling the term $P_{t+\tau}(\hat{y}|y,{\mathbf x}_t \in \mathfrak{I}_k^t)$ 
as a projection into the future of the probability distribution over $\mathfrak{I}_k^{t+\tau}$, the confidence intervals of $h_{t + \tau}$. %As illustrated by Figure \ref{fig-confidence-intervals-new}, 
A Markov-chain model is used  for this purpose. A fully detailed description of the \textsc{Economy}-$\gamma$ is provided in (\cite{achenchabe2021early}).

% \begin{figure}[htbp!]
% \centering
% \includegraphics[width=0.75\linewidth]{./figures/fig-confidence-intervals-new.pdf}
% \caption{\textsc{Economy-$\gamma$}, computing the probability distribution $p(\gamma_{t+\tau}|\gamma_t)$. Here $h_t(\x_t)$ falls in the second confidence level interval. Given a supposed learned transition matrix $\Mat_{t}^{t+1}$, the next vector of confidence levels will be 
% ${(0.15, 0.3, 0.3, 0.2, 0.05)^{\top}}$. 
% }
% \label{fig-confidence-intervals-new}
% \end{figure} 

%----------------------------------------------------------------------------------------------------------
\subsection{Data and feature extraction}
\label{sec_data}
%----------------------------------------------------------------------------------------------------------

Because \textsc{Economy}-$\gamma$ is restricted to binary classification problems, and in order to be able to directly compare our results with those reported in (\cite{achenchabe2021early}), we chose to use the same 34 datasets
that are taken from the UEA \& UCR Time Series Classification Repository\footnote{Available at: http://www.timeseriesclassification.com} (\cite{bagnall16bakeoff}). However, it is important to note that the revocable framework presented here could as well accommodate multi-class classification problems. %, for instance using the \textsc{Economy}-K method instead of \textsc{Economy}-$\gamma$. 
Additional experiments on multi-class problem are reported in the supplementary material % \cite{GIT} TODO.

Each training set is built with 70\% of the examples randomly uniformly selected, while the remaining 30\% are used as test set (note that in each dataset, all time series have the same length).
In addition, each training set is divided into three disjoint subsets: (\textit{i}) 40\% for training the Xgboost   (\cite{chen_xgboost_2016}) classifiers $\{h_t\}_{t \in \{1, \ldots, T\}}$ that are the base classifiers used in the \textsc{Economy}-$\gamma$ method, which offer a good trade-off between computing time and accuracy; (\textit{ii}) 40\% for estimating the probabilities in $f^{rev}_{\tau}$ and $f_{\tau}$; and (\textit{iii}) the remaining 20\% for optimizing the number of groups $|\G|$ in \textsc{Economy}-$\gamma$ which is its only hyper-parameter.

In order to give equal weight to all data sets in the comparison, it is important that they offer the same number of opportunities for decision changes. This is why instants for potential changes are sampled every $n$\% of the length of the times series in each data set (in our case, $n$= 5\%). 
For each possible length, 60 features on the statistical, temporal and spectral domains are extracted using the Time Series Feature Extraction Library (\cite{tsfel}), and are used for training the classifiers $\{h_t\}_{t \in \{1, \ldots, T\}}$.

%----------------------------------------------------------------------------------------------------------
\subsection{The evaluation criterion}
%----------------------------------------------------------------------------------------------------------

The cost incurred using an early classification system on a time series ${\mathbf x}_T$ is the sum of three costs, the cost of misclassification, the delay cost incurred at the time of the last decision, and the sum of the costs associated with all changes of decision if any:
\begin{equation}
%\small
\text{Cost}({\mathbf x}_T) \;  = \;   
\mathrm{C}_m  ( h_{t_{l}}({\mathbf x_{t_{l}}})|y)  
+   \mathrm{C}_d(t_{l}) 
+ \sum_{i=1}^{ |\D_{l}|-1 } C_{cd}(\hat{y}_{i+1}|\hat{y}_{i})
\label{eq:eval-individual-cost}
\end{equation}

In order to evaluate a method, we compute its mean performance on the test set ${\cal T}$:
\begin{equation}
    AvgCost({\cal T}) \; = \; \frac{1}{|{\cal T}|} \, \sum_{i=1}^{|{\cal T}|} \, \text{Cost}({\mathbf x}^i_T)
\label{eq:eval-cost}
\end{equation}

%----------------------------------------------------------------------------------------------------------
\subsection{Description of the experiments}
\label{exp-descr}
%----------------------------------------------------------------------------------------------------------

In our experiments, we compared three algorithms: \textsc{Economy}-$\gamma$ which is an irrevocable decision-maker
%\footnote{\textcolor{red}{To our knowledge the present paper suggests the first revocable method and thusit is hard to find competitors for a new problem. Therefore we decided to test the method against the best irrevocable early classification method we found from the state of the art \cite{achenchabe2021early}. But other competitors could be added in future work.}}
, \textsc{eco-rev-cu} which is the revocable version of \textsc{Economy}-$\gamma$ but unaware of the costs of changing decision, and, finally, \textsc{eco-rev-ca} the revocable decision-maker that is aware of these changing costs. We are thus able to measure the added-value of the revocable strategy (\textsc{eco-rev-cu} \textit{vs.} \textsc{Economy}-$\gamma$) and the added-value of being aware of the costs of changing decision %wrt; not being aware 
(\textsc{eco-rev-ca} \textit{vs.} \textsc{eco-rev-cu}).

For a given application, it is expected that the various costs, relative to misclassifications, delays and changes of decision, will be provided by the domain expert. For our experiments, we explored the performance of the three methods on a wide range of cost values:
\begin{itemize}
    \item The \textit{misclassification cost} was set to $\mathrm{C}_m (\hat{y} |y)= 1$ if ~$\hat{y} \neq y$, and $= 0$ if not. 
    \item The \textit{delay cost} was assumed to be linear with a positive slope: $C_d = \alpha \times \frac{t}{T}$ starting from very low $\alpha$ = \{ 0.0001, 0.00025, 0.0005, 0.00075\}, to low $\alpha$ =  \{0.001, 0.0025, 0.005, 0.0075\}, to medium values $\alpha$ =  \{0.01, 0.025, 0.05 ,0.075\} and to high values $\alpha$ =  \{0.1, 0.25, 0.5, 0.75, 1\}.
    \item The \textit{cost of changing decision} was set to $\mathrm{C}_{cd} (\hat{y_1} |\hat{y_2})= \beta$ if $\hat{y_1} \neq \hat{y_2}$, and $= 0$ otherwise. The parameter $\beta$ being taken in the same set of values as $\alpha$\footnote{$\alpha$ and $\beta$ were chosen in a very large spectrum of values so as not biasing the results.} 
\end{itemize}

 The AvgCost criterion defined in Equation \ref{eq:eval-cost} was evaluated on the 34 test sets for all cost values, and the Wilcoxon signed-rank test (\cite{10.2307/3001968}) was performed for all the range of cost values, in order to assess whether the observed performance gap between methods is significant \textit{(``+'' and ``-'')} or not \textit{(``$\circ$'')}.

%----------------------------------------------------------------------------------------------------------
\subsection{Results and analysis}
%----------------------------------------------------------------------------------------------------------

Before comparing the methods, it is important to measure the proportion of time series that offer useful opportunities for revocable decisions. Those are the ones where the first decision taken by an irrevocable strategy, here \textsc{Economy}-$\gamma$, turns out not to be optimal. 
For the 34 datasets under study, it turns out that (\textit{i}) for a low delay cost $C_d = 0.0025 \times \frac{t}{T}$ only 3\% of the first decisions can be usefully revoked; (\textit{ii}) for a medium delay cost $C_d = 0.025 \times \frac{t}{T}$ this percentage rises to 3.6\%; and (\textit{iii}) for a high delay cost $C_d = 0.5 \times \frac{t}{T}$ this percentage reaches 8\%.
These figures show that, for these datasets and this range of cost values, it is not clear that a significant added-value can be found in favor of the revocable strategies. 
That such an advantage is nonetheless what the experiments bring out is thus remarkable (see Supplementary Material for a more detailed analysis).
%\textcolor{red}{Mets-on le tableau qui était en annexe, ou alors on peut aussi refaire une annexe?}

%\textcolor{blue}{
%However, this is what the experiments demonstrate, which is thus remarkable.
%}

The first lesson is that both revocable methods \textsc{eco-rev-cu} and \textsc{eco-rev-ca} get significantly better results than the irrevocable method \textsc{Economy}-$\gamma$ on a wide range of delay cost $\mathrm{C}_d$ and decision change cost values $\beta$ (see Figures \ref{fig:eco_comp}(a) and \ref{fig:eco_comp}(b)). The second lesson is that it pays off to use a strategy which takes into account the costs of changing decision. Indeed, \textsc{eco-rev-ca} beats \textsc{Economy}-$\gamma$ on a wider range of conditions than \textsc{eco-rev-cu}. 
%\textcolor{blue}{
%And \textsc{eco-rev-ca} get significantly better results than \textsc{eco-rev-cu} on a range of cost values (see Figure \ref{fig:eco_comp}(c)).
%}

Both revocable strategies fail to overcome the irrevocable one, \textsc{Economy}-$\gamma$, when $\beta$ is large (i.e. more than $0.1$), 
and then \textsc{eco-rev-cu} fails more often than \textsc{eco-rev-ca}. 
This behavior is not surprising since, when it is very costly to delay a decision, the best strategy is generally to make a very early decision and not to revise it afterwards. 

\begin{figure}[htbp!]
\centering
 \includegraphics[width=0.7\linewidth]{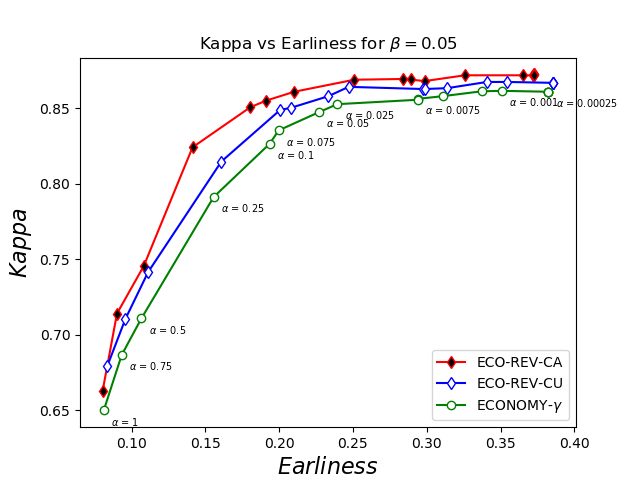}
\caption{Average \textit{Earliness} vs. Average \textit{Kappa} score obtained over all the $34$ datasets for $\beta = 0.05$ and by varying the slope $\alpha$ of the delay cost. The reader may find the same behavior for other $\beta$ values in the supplementary material.}
\label{fig:pareto}
\vspace{-4mm}
\end{figure}

Figure \ref{fig:eco_comp}(c) shows the results of the Wilcoxon signed-rank test between the two revocable strategies. % that we propose in this paper, 
It appears that the cost aware approach \textsc{eco-rev-ca} performs significantly better than the cost unaware approach \textsc{eco-rev-cu}, for almost one third of the pairs of values ($\alpha$, $\beta$). 
As the slope of the delay cost $\alpha$ grows, \textsc{eco-rev-ca} becomes significantly better than \textsc{eco-rev-cu} for an increasing larger range of values for $\beta$. %for bigger values of $\mathrm{C}_{cd}$. 
This means that when the delay cost is rather high, it pays off to use a revocable strategy that takes into account the cost of changing decision. 
In addition, the Friedman test (\cite{Nemenyi62}) shows that \textsc{eco-rev-ca} is on average better ranked than \textsc{eco-rev-cu} in 96\% of pairs ($\alpha$, $\beta$). (Further details are available in the the supplementary material).

%%%%%%%%%%%%%%%%%%%%%%%%%%%%%%%%%%%

\begin{figure}[htbp!]
\centering
(a){\label{fig:ECO-REV-CUvsGAMMA} \includegraphics[width=0.70\linewidth]{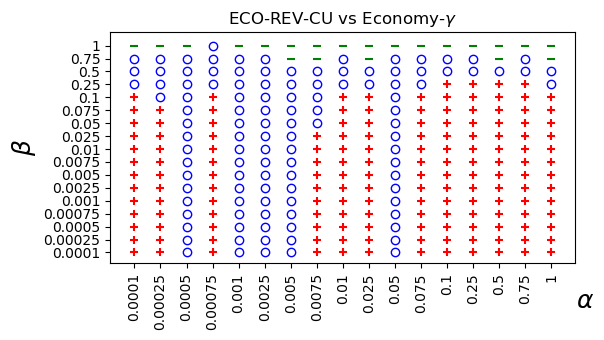}} \\

(b){\label{fig:ECO-REV-CAvsGAMMA} \includegraphics[width=0.70\linewidth]{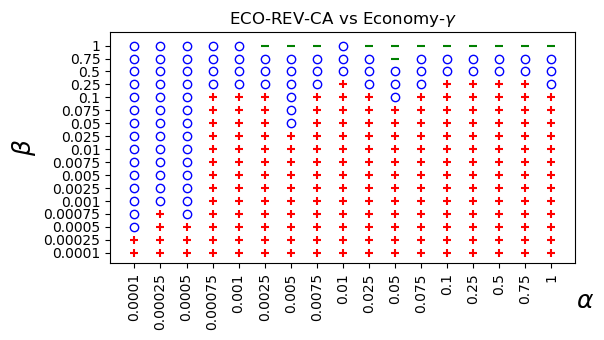}} \\

(c){\label{wilco_ca_cu} \includegraphics[width=0.70\linewidth]{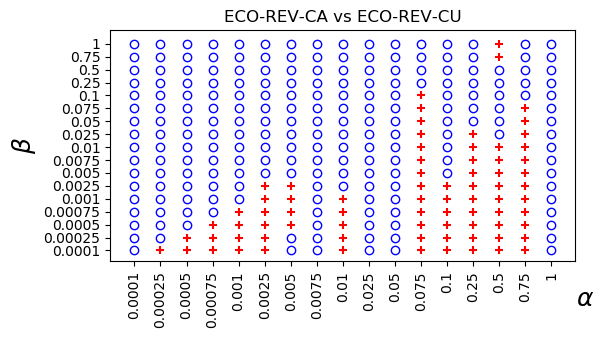}}

\caption{(a) \textsc{eco-rev-cu} \textit{vs.} \textsc{Economy}-$\gamma$; (b) \textsc{eco-rev-ca} \textit{vs.} \textsc{Economy}-$\gamma$; (c) \textsc{eco-rev-ca} \textit{vs.} \textsc{eco-rev-cu}. Wilcoxon signed-rank test applied on the $AvgCost$ criterion over the $34$ test sets, for a range of couples of values $\alpha$ and $\beta$, with ``{$+$}'' indicating a significant success of the first approach, ``{$\circ$}'' an insignificant difference and ``{$-$}'' indicating a significant failure of the first approach.}
\label{fig:eco_comp}
\end{figure}

%%%%%%%%%%%%%%%%%%%%%%%%%%%%%%%%%%%

In order to get a global view of the merits of each method, we have drawn \textit{Pareto curves} (see Figure \ref{fig:pareto}) with respect to the average Cohen's \textit{kappa} score (\cite{cohen1960coefficient}) and the average \textit{earliness}, 
which is defined as the mean of the last triggering times normalized by the length of the series $ earliness = {Avg\{{t_\ell} / {T}\}}$. 
%(\textcolor{red}{Pas clair : c'est quoi cette moyenne ? La moyenne du dernier changement de décision sur tous les datasets et time series (j'espère) ou la moyenne des changements de décision comme c'est dit ?}). 
These two quantities are averaged over the $34$ datasets by varying $\alpha$ in the range of values defined in Section \ref{exp-descr}, and for $\beta = 0.05$. The Pareto curves that are obtained show that 
%(\textit{i}) Algorithm \ref{algo:algo1} \textcolor{red}{(c'est quoi cet algo ? Pas référencé dans le texte.)} dominates all the approaches since it has access to the ground truth; 
(\textit{i}) the baseline irrevocable \textsc{Economy}-$\gamma$ method is dominated by the two revocable strategies; 
(\textit{ii}) and \textsc{eco-rev-ca} dominates \textsc{eco-rev-cu}. 
These results are consistent with the ones cited above and hold for other values of $\beta$ a well (see the supplementary material %\cite{GIT} TODO 
for more details). More finely, it is apparent that, as the slope $\alpha$ of the delay cost increases, from 0.00025 to 1, all methods first maintain a high kappa, before being unable to maintain it as they are forced to make decisions too early. Still, the \textsc{eco-rev-ca} algorithm is the one that best resists. 

Overall, our experiments show the interest of using revocable strategies for the early classification of time series in a wide range of delay and change of decision costs. They demonstrate also the relevance of the formal criterion that we propose for these strategies.

%----------------------------------------------------------------------------------------------------------
%----------------------------------------------------------------------------------------------------------

%\vspace{-1mm}
\section{Conclusion}
\label{sec_conclusion}
%----------------------------------------------------------------------------------------------------------
%----------------------------------------------------------------------------------------------------------
%\vspace{-2mm}

Applications abound, in which incoming time series must be labeled as early and as accurately as possible, before all measurements are available. Until now, the problem of early classification of time series was addressed by triggering irrevocable decisions. 
For the first time, this paper defines the revocable version of this problem and introduces an associated optimization problem. Optimization is performed over the space of all possible sequences of decisions given an incoming time series. 
Thanks to a non-myopic criterion given in the paper its exploration is simplified and an algorithm follows naturally. This algorithm has been implemented in two versions, the first one takes into account the cost of changing the decision and the second one does not. Extensive experiments have shown that this algorithm, which explicitly takes into account the cost of changing decisions, has significantly better performances than the same algorithm that does not. In addition, comparison with the irrevocable regime (\textsc{Economy}-$\gamma$) shows that the two proposed algorithms make useful revocations, because they both outperform the irrevocable regime.

As future work, we plan to adapt the proposed framework to the online setting, where decisions should be taken based on a data stream. 
%improving  %in order to scale with a large number of class values. 
We will also improve the multi-class approaches proposed in the supplementary material. % \cite{GIT} TODO.  

%%%%%%%%%%%%%%%%%%%%%%%%%%%%%%%%%%%%%%%%%%%%%%%%%%%%%%%

\bibliography{myrefs}
%An example of citation~\cite{DBLP:conf/acml/2009}.

%\acks{Acknowledgements should go at the end, before appendices and references.}

%\bibliography{acml21}

%\appendix

%\section{First Appendix}\label{apd:first}

%This is the first appendix.

%\section{Second Appendix}\label{apd:second}

%This is the second appendix.

\end{document}